\documentclass{article}

\usepackage{microtype}
\usepackage{graphicx}
\usepackage{subfigure}
\usepackage{booktabs} 

\usepackage{hyperref}
\usepackage{hyperref}
\usepackage{amsfonts}
\usepackage{amsmath}
\usepackage{amsthm}
\usepackage{enumitem}

\newtheorem{theorem}{Theorem}
\newtheorem{definition}{Definition}
\newtheorem{corollary}{Corollary}
\newcommand\numberthis{\addtocounter{equation}{1}\tag{\theequation}}

\usepackage{cleveref}

\RequirePackage[prologue]{xcolor}
\definecolor[named]{ACMBlue}{cmyk}{1,0.1,0,0.1}
\definecolor[named]{ACMYellow}{cmyk}{0,0.16,1,0}
\definecolor[named]{ACMOrange}{cmyk}{0,0.42,1,0.01}
\definecolor[named]{ACMRed}{cmyk}{0,0.90,0.86,0}
\definecolor[named]{ACMLightBlue}{cmyk}{0.49,0.01,0,0}
\definecolor[named]{ACMGreen}{cmyk}{0.20,0,1,0.19}
\definecolor[named]{ACMPurple}{cmyk}{0.55,1,0,0.15}
\definecolor[named]{ACMDarkBlue}{cmyk}{1,0.58,0,0.21}
\definecolor{darkgreen}{rgb}{0,0.5,0}

\usepackage{tikz}
\usepackage{threeparttable}
\usepackage{multirow}

\newcommand{\fct}{fairness--confusion tensor}
\newcommand{\PFOP}{performance--fairness optimality problem}
\newcommand{\LAFOP}{least-squares accuracy--fairness optimality problem}
\newcommand{\MLAFOP}{multilinear least-squares accuracy--fairness optimality problem}

\DeclareMathOperator*{\argmin}{arg\,min}
\DeclareMathOperator*{\sgn}{sgn}
\newcommand{\Prob}{{\sf Pr}}

\newcommand{\simplex}{\mathcal K}

\renewcommand{\iff}{if and only if}

\newcommand{\gaussian}{\ensuremath{\mathcal N}}
\newcommand{\bernoulli}{\ensuremath{\mathcal B}}

\newcommand{\twoval}{\ensuremath{\{0, 1\}}}
\newcommand{\fairvar}{\ensuremath{a}}
\newcommand{\predvar}{\ensuremath{\hat y}}
\newcommand{\truthvar}{\ensuremath{y}}
\newcommand{\cntvar}{\ensuremath{N}}
\newcommand{\posvar}{\ensuremath{M}}
\newcommand{\lmvar}{\ensuremath{\lambda}}
\newcommand{\secondlmvar}{\ensuremath{\mu}}

\newcommand{\rankvar}{\ensuremath{\rho}}

\newcommand{\fairset}{\ensuremath{\Phi}}

\newcommand{\fairepsvar}{\ensuremath{\epsilon}}
\newcommand{\perfepsvar}{\ensuremath{\delta}}

\newcommand{\fctvar}{\ensuremath{\mathbf z}}

\newcommand{\attrvar}{\ensuremath{\mathbf a}}
\newcommand{\data}{\ensuremath{\mathbf x}}
\newcommand{\baserate}{\ensuremath{\alpha}}

\newcommand{\perfvec}{\ensuremath{\mathbf c}}
\newcommand{\fairmat}{\ensuremath{\mathbf A}}

\newcommand{\fairqfmat}{\ensuremath{\mathbf B}}

\newcommand{\fairfn}{\ensuremath{\phi}}
\newcommand{\perffn}{\ensuremath{f}}

\newcommand{\marginalmat}{\ensuremath{\mathbf{A}_{\sf const}}}
\newcommand{\marginalvec}{\ensuremath{\mathbf{b}_{\sf const}}}

\newcommand{\name}{\textsc{FACT}}

\newcommand{\DPpic}{
\begin{tikzpicture}
\draw[help lines, step=4pt] (0,0) grid (8pt, 8pt);     
\draw[help lines, step=4pt] (12pt,0) grid (20pt, 8pt); 
\fill[blue!100!white] ( 0pt,4pt) rectangle ( 4pt,8pt);  
\fill[blue!100!white] ( 4pt,4pt) rectangle ( 8pt,8pt);  
\fill[blue! 30!white] ( 0pt,0pt) rectangle ( 4pt,4pt);  
\fill[blue! 30!white] ( 4pt,0pt) rectangle ( 8pt,4pt);  
\fill[ red!100!white] (12pt,4pt) rectangle (16pt,8pt);  
\fill[ red!100!white] (16pt,4pt) rectangle (20pt,8pt);  
\fill[ red! 30!white] (12pt,0pt) rectangle (16pt,4pt);  
\fill[ red! 30!white] (16pt,0pt) rectangle (20pt,4pt);  
\end{tikzpicture}
}

\newcommand{\REOdPic}{
\begin{tikzpicture}
\draw[help lines, step=4pt] (0,0) grid (8pt, 8pt);     
\draw[help lines, step=4pt] (12pt,0) grid (20pt, 8pt); 
\fill[blue! 30!white] ( 0pt,4pt) rectangle ( 4pt,8pt);  
\fill[blue!100!white] ( 4pt,4pt) rectangle ( 8pt,8pt);  
\fill[blue!100!white] ( 0pt,0pt) rectangle ( 4pt,4pt);  
\fill[blue! 30!white] ( 4pt,0pt) rectangle ( 8pt,4pt);  
\fill[ red!30!white] (12pt,4pt) rectangle (16pt,8pt);  
\fill[ red!100!white] (16pt,4pt) rectangle (20pt,8pt);  
\fill[ red!100!white] (12pt,0pt) rectangle (16pt,4pt);  
\fill[ red! 30!white] (16pt,0pt) rectangle (20pt,4pt);  
\end{tikzpicture}
}

\newcommand{\EOppic}{
\begin{tikzpicture}
\draw[help lines, step=4pt] (0,0) grid (8pt, 8pt);     
\draw[help lines, step=4pt] (12pt,0) grid (20pt, 8pt); 
\fill[blue!100!white] ( 0pt,4pt) rectangle ( 4pt,8pt);  
\fill[blue! 30!white] ( 0pt,0pt) rectangle ( 4pt,4pt);  
\fill[ red!100!white] (12pt,4pt) rectangle (16pt,8pt);  
\fill[ red! 30!white] (12pt,0pt) rectangle (16pt,4pt);  
\end{tikzpicture}
}

\newcommand{\PEpic}{
\begin{tikzpicture}
\draw[help lines, step=4pt] (0,0) grid (8pt, 8pt);     
\draw[help lines, step=4pt] (12pt,0) grid (20pt, 8pt); 
\fill[blue!100!white] ( 4pt,4pt) rectangle ( 8pt,8pt);  
\fill[blue! 30!white] ( 4pt,0pt) rectangle ( 8pt,4pt);  
\fill[ red!100!white] (16pt,4pt) rectangle (20pt,8pt);  
\fill[ red! 30!white] (16pt,0pt) rectangle (20pt,4pt);  
\end{tikzpicture}
}

\newcommand{\PPpic}{
$($\nolinebreak[4]
\begin{tikzpicture}
\draw[help lines, step=4pt] (0,0) grid (8pt, 8pt);     
\draw[help lines, step=4pt] (12pt,0) grid (20pt, 8pt); 
\fill[blue!100!white] ( 0pt,4pt) rectangle ( 4pt,8pt);  
\fill[blue!100!white] ( 4pt,4pt) rectangle ( 8pt,8pt);  
\fill[ red!100!white] (12pt,4pt) rectangle (16pt,8pt);  
\fill[ red!100!white] (16pt,4pt) rectangle (20pt,8pt);  
\end{tikzpicture}
\nolinebreak[4]
$)^{(2)}$
}

\newcommand{\EFORpic}{
$($\nolinebreak[4]
\begin{tikzpicture}
\draw[help lines, step=4pt] (0,0) grid (8pt, 8pt);     
\draw[help lines, step=4pt] (12pt,0) grid (20pt, 8pt); 
\fill[blue!100!white] ( 0pt,0pt) rectangle ( 4pt,4pt);  
\fill[blue!100!white] ( 4pt,0pt) rectangle ( 8pt,4pt);  
\fill[ red!100!white] (12pt,0pt) rectangle (16pt,4pt);  
\fill[ red!100!white] (16pt,0pt) rectangle (20pt,4pt);  
\end{tikzpicture}
\nolinebreak[4]
$)^{(2)}$
}

\newcommand{\EFNRpic}{
\begin{tikzpicture}
\draw[help lines, step=4pt] (0,0) grid (8pt, 8pt);     
\draw[help lines, step=4pt] (12pt,0) grid (20pt, 8pt); 
\fill[blue!30!white] ( 0pt,4pt) rectangle ( 4pt,8pt);  
\fill[blue!100!white] ( 0pt,0pt) rectangle ( 4pt,4pt);  
\fill[ red!30!white] (12pt,4pt) rectangle (16pt,8pt);  
\fill[ red! 100!white] (12pt,0pt) rectangle (16pt,4pt);  
\end{tikzpicture}
}

\newcommand{\CGpic}{
\begin{tikzpicture}
\draw[help lines, step=4pt] (0,0) grid (8pt, 8pt);     
\draw[help lines, step=4pt] (12pt,0) grid (20pt, 8pt); 
\fill[blue!100!white] ( 0pt,4pt) rectangle ( 4pt,8pt);  
\fill[blue!100!white] ( 4pt,4pt) rectangle ( 8pt,8pt);  
\end{tikzpicture}
$\land$
\begin{tikzpicture}
\draw[help lines, step=4pt] (0,0) grid (8pt, 8pt);     
\draw[help lines, step=4pt] (12pt,0) grid (20pt, 8pt); 
\fill[blue!100!white] ( 0pt,0pt) rectangle ( 4pt,4pt);  
\fill[blue!100!white] ( 4pt,0pt) rectangle ( 8pt,4pt);  
\end{tikzpicture}
$\land$
\begin{tikzpicture}
\draw[help lines, step=4pt] (0,0) grid (8pt, 8pt);     
\draw[help lines, step=4pt] (12pt,0) grid (20pt, 8pt); 
\fill[ red!100!white] (12pt,4pt) rectangle (16pt,8pt);  
\fill[ red!100!white] (16pt,4pt) rectangle (20pt,8pt);  
\end{tikzpicture}
$\land$
\begin{tikzpicture}
\draw[help lines, step=4pt] (0,0) grid (8pt, 8pt);     
\draw[help lines, step=4pt] (12pt,0) grid (20pt, 8pt); 
\fill[ red!100!white] (12pt,0pt) rectangle (16pt,4pt);  
\fill[ red!100!white] (16pt,0pt) rectangle (20pt,4pt);  
\end{tikzpicture}
}

\newcommand{\CGmat}{
\begin{pmatrix}
1-v_1 & 0     & -v_1 &    0 & 0     & 0     &    0 & 0\\
0     & 1-v_0 &    0 & -v_0 & 0     & 0     &    0 & 0\\
0     & 0     &    0 &    0 & 1-v_1 & 0     & -v_1 & 0\\
0     & 0     &    0 &    0 & 0     & 1-v_0 &    0 & -v_0
\end{pmatrix}
}

\newcommand{\PCBpic}{
\begin{tikzpicture}
\draw[help lines, step=4pt] (0,0) grid (8pt, 8pt);     
\draw[help lines, step=4pt] (12pt,0) grid (20pt, 8pt); 
\fill[blue!100!white] ( 0pt,4pt) rectangle ( 4pt,8pt);  
\fill[blue! 30!white] ( 0pt,0pt) rectangle ( 4pt,4pt);  
\fill[ red!100!white] (12pt,4pt) rectangle (16pt,8pt);  
\fill[ red! 30!white] (12pt,0pt) rectangle (16pt,4pt);  
\end{tikzpicture}
}

\newcommand{\NCBpic}{
\begin{tikzpicture}
\draw[help lines, step=4pt] (0,0) grid (8pt, 8pt);     
\draw[help lines, step=4pt] (12pt,0) grid (20pt, 8pt); 
\fill[blue!100!white] ( 4pt,4pt) rectangle ( 8pt,8pt);  
\fill[blue! 30!white] ( 4pt,0pt) rectangle ( 8pt,4pt);  
\fill[ red!100!white] (16pt,4pt) rectangle (20pt,8pt);  
\fill[ red! 30!white] (16pt,0pt) rectangle (20pt,4pt);  
\end{tikzpicture}
}

\usepackage[accepted]{icml2020}

\icmltitlerunning{FACT: A Diagnostic for Group Fairness Trade-offs}

\begin{document}

\twocolumn[
\icmltitle{FACT: A Diagnostic for Group Fairness Trade-offs}




\begin{icmlauthorlist}
\icmlauthor{Joon Sik Kim}{cmu,intern}
\icmlauthor{Jiahao Chen}{jpm}
\icmlauthor{Ameet Talwalkar}{cmu,determined}
\end{icmlauthorlist}

\icmlaffiliation{cmu}{Machine Learning Department, Carnegie Mellon University, Pittsburgh, USA}
\icmlaffiliation{intern}{Work paritally done during an internship at JP Morgan}
\icmlaffiliation{jpm}{JP Morgan AI Research, New York, USA}
\icmlaffiliation{determined}{Determined AI, San Francisco, USA}

\icmlcorrespondingauthor{Joon Sik Kim}{joonkim@cmu.edu}

\icmlkeywords{Machine Learning, ICML}

\vskip 0.3in
]



\printAffiliationsAndNotice{}  

\begin{abstract}
Group fairness, a class of fairness notions that measure how different groups of individuals are treated differently according to their protected attributes, has been shown to conflict with one another, often with a necessary cost in loss of model's predictive performance.
We propose a general diagnostic that enables systematic characterization of these trade-offs in group fairness. 
We observe that the majority of group fairness notions can be expressed via the fairness--confusion tensor, which is the confusion matrix split according to the protected attribute values. We frame several optimization problems that directly optimize both accuracy and fairness objectives over the elements of this tensor, which yield a general perspective for understanding multiple trade-offs including group fairness incompatibilities.
It also suggests an alternate post-processing method for designing fair classifiers.
On synthetic and real datasets, we demonstrate the use cases of our diagnostic, particularly on understanding the trade-off landscape between accuracy and fairness. 
\end{abstract}

\section{Introduction}

As machine learning continues to be more widely used for applications with societal impact such as credit decisioning, predictive policing, and employment applicant screening, practitioners face regulatory, ethical, and legal challenges to prove whether or not their models are fair~\cite{ainow2019}.
To provide quantitative tests of model fairness, the practitioners further need to choose between multiple definitions of fairness that exist in the machine learning literature~\cite{calders2009indep, zliobaite2015relation,narayanan2018translation}. Among them is a class of definitions called \emph{group fairness}, which measures how a group of individuals with certain protected attributes are treated differently from other groups. This notion is widely studied as a concept of \emph{disparate impact} in the legal context, and one specific instance of this notion was enforced as a law for fair employment process back in 1978~\cite{biddle2006adverse}. From a technical point of view however, several notions of group fairness have been shown to conflict with one another~\cite{kleinberg2016inherent, chouldechova2017fair}, sometimes with a necessary cost in loss of accuracy~\cite{liu2019cg}. 
Such considerations complicate the practical development and assessment of machine learning models designed to satisfy group fairness, as the conditions under which these trade-offs must necessarily occur can be too abstract to understand. Previous works on these trade-offs have been presented in ad hoc and definition-specific manner, which further calls for a more general perspective addressing the trade-offs in practice. 

As an example, suppose an engineer is responsible for training a loan prediction model from a large user dataset,
subject to mandatory group fairness requirements shaped by regulatory concerns. 
She has many choices for how to train this fair model,
with fairness enforced before~\cite{kamiran2010discrimination, zemel2013learning, madras2018learning, samadi2018price, song2019learning, tan2019learning}, during~\cite{zafar2015fairness, zafar2017fairness}, or after~\cite{dwork2012fairness, feldman2015di, hardt2016equality} training.
However, she must resort to trial and error to determine which of these myriad approaches, if any, will produce a compliant model with sufficient
performance\footnote{In this work, \emph{performance} refers to classical metrics derived from the confusion matrix, e.g., accuracy, precision and fairness notions are not part of it.}
to satisfy business needs.
It may even turn out that despite her best efforts,
the fairness constraints set by the regulators are actually impossible to satisfy to begin with,
due to limitations intrinsic to the prediction task and data at hand. If there was a tool to understand the potential trade-offs exhibited by the model, even before training, it would be easier for multiple parties to effectively reconcile the conflicting components in designing fair classifiers. 

Motivated by such practical considerations, we propose the \emph{FACT (\textbf{FA}irness-\textbf{C}onfusion \textbf{T}ensor) diagnostic}
for exploring the trade-offs involving group fairness: the diagnostic provides a general framework under which the practitioners can understand both fairness--fairness trade-offs and fairness--performance trade-offs. At the core of our diagnostic lies the \emph{fairness--confusion tensor}, which is the confusion matrix divided along an additional axis for protected attributes. The FACT diagnostic first expresses the majority of group fairness notions as linear/quadratic functions of the elements of this tensor. The simplicity of these functions makes it easy for them to be naturally integrated into a class of optimization problems over the elements of the tensor (not over the model parameters), which we call \emph{\PFOP{}} (PFOP). It essentially considers the geometry of valid \fct{}s that satisfy a specified set of performance and/or fairness conditions. 

By noting that many settings involve only linear notions of fairness, in this work we focus on \emph{\LAFOP{}} (LAFOP) and \emph{model-specific \LAFOP{}} (MS-LAFOP), which are specific instantiations of PFOP, each representative of model-agnostic and model-specific scenarios. In particular, for the model-agnostic case, the diagnostic allows for a comparative analysis of the \emph{relative} difficulty of learning a classifier under additional group fairness constraints imposed. This difficulty is interpreted with respect to the Bayes error, which is the inherent difficulty of the fairness-unconstrained learning problem, hence a natural reference point.


Our contributions are:
\begin{enumerate}
\item to demonstrate how \fct{} characterizes the majority of group fairness definitions in the literature as linear or quadratic functions, whose simplicity can be leveraged to formulate optimization problems suited for trade-off analysis,

\item to formulate the \name\ diagnostic as a PFOP, LAFOP, and MS-LAFOP over the \fct{}, enabling both model-agnostic and model-specific analysis of fairness trade-offs,

\item to provide a general understanding of group fairness incompatibility, which simplifies the existing results in the literature and extends them to new types,

\item to demonstrate the use of the FACT diagnostic on synthetic and real datasets, e.g. how it can be used for diagnosis of relative influence of the fairness notions on performance and other fairness conditions, and how it can be used as a post-processing method for designing fair classifiers.
\end{enumerate}

\section{Related Work}%
\label{sec:relatedwork}

\textbf{Fairness--confusion tensor} is not a completely new notion -- several work has implicitly mentioned it, mostly disregarding it as a simple computational tool that eases the computation on an implementation level~\cite{bellamy2018ai, celis2019classification}. It is also a natural object considered in several post-processing methods in fairness~\cite{hardt2016equality, pleiss2017fairness}, a group of algorithms that fine-tune a trained model to mitigate the unfairness while keeping the performance change minimal. Here we take a closer look at the \fct{} itself and study how this object naturally brings together several notions of group fairness, simplifying and generalizing the analysis of inherent trade-offs within. 

\textbf{Quantitative definitions of group fairness} exist in many different variations~\cite{narayanan2018translation,  kleinberg2016inherent, chouldechova2017fair, dwork2012fairness, hardt2016equality,  calders2010dp,  berk2018fairness} but few work exists to  categorize these notions with a broader perspective encompassing the trade-off schemes. 
\citet{verma2018fairness} categorized the existing group fairness definitions based on entries and rates derived from the fairness--confusion tensor but did not explore any trade-offs and incompatibilities within.
Our work extends this effort and provides a versatile geometric formalism to study the trade-offs.

\textbf{Fairness--performance trade-offs} have been studied in many specific cases~\cite{calders2009indep,  zliobaite2015relation,kamiran2010discrimination, feldman2015di, menon2018cost, liu2019cg, zhao2019inherent},
for limited definitions of fairness, performance, and models. To our knowledge, these trade-offs have not been studied in the general way we present below.
\citet{zafar2015fairness, zafar2017fairness} presented an optimization-based analysis of the trade-offs,
albeit over the parameter space of a particular model.

\textbf{Fairness--fairness trade-offs} describe the incompatibility of multiple notions of group fairness~\cite{kleinberg2016inherent,chouldechova2017fair,pleiss2017fairness,berk2018fairness} without some strong assumptions about the data and the model. Previous incompatibility results have been presented mostly in ad hoc and definition-specific manner, which our diagnostic addresses with a more general perspective for understanding incompatibilities. We show a general incompatibility result involving Calibration fairness condition, which naturally implies the result in \citet{kleinberg2016inherent} along with many other new ones.  
To the best of our knowledge,
our work is the first to provide a systematic approach to diagnose both fairness--fairness and fairness--performance trade-offs together for group fairness under the same formalism.

\setlength\arraycolsep{2pt}

\section{The Fairness--confusion Tensor}
\label{sec:linear}

\begin{table*}[!ht]
\small
\resizebox{\textwidth}{!}{
\begin{threeparttable}[t]
    \centering
    \begin{tabular}{l l l}
Name of fairness & Definition and linear system & Terms in \fct \\
\toprule
Demographic parity (DP) &
$\Prob(\hat{y} = 1 | \attrvar = 1) = \Prob(\hat{y} = 1  | \attrvar = 0)$
\\ &
$\fairmat_\textsc{dp} = \frac 1 \cntvar \begin{pmatrix}
\cntvar_0 & 0 & \cntvar_0 & 0 & -\cntvar_1 & 0 & -\cntvar_1 & 0
\end{pmatrix}$ &
\DPpic
\\
Equality of opportunity (EOp)\cite{hardt2016equality} &
$\Prob(\hat{y} = 1 | y=1, \attrvar = 1) = \Prob(\hat{y} = 1| y = 1 , \attrvar = 0)$ &
\\ &
$\fairmat_\textsc{eop} = \frac 1 \cntvar \begin{pmatrix}
\posvar_0 & 0 & 0 & 0 & -\posvar_1 & 0 & 0 & 0
\end{pmatrix}
$ &
\EOppic
\\

Predictive equality (PE)\cite{chouldechova2017fair} &
$\Prob(\hat{y} = 1 | y=0 , \attrvar = 1) = \Prob(\hat{y} = 1 |y=0 , \attrvar = 0)$ &
\\ &
$\fairmat_\textsc{pe} = \frac 1 \cntvar \begin{pmatrix}
0 & 0 & \cntvar_0 - \posvar_0 & 0 & 0& 0 & -\cntvar_1 + \posvar_1&0
\end{pmatrix}$ &
\PEpic
\\

Equalized odds (EOd)\cite{hardt2016equality} &
EOp $\land$ PE &
\EOppic $\land$ \PEpic
\\


Equal false negative rate (EFNR) \tnote{2} &
$\Prob(\predvar = 0 | \truthvar = 1, \attrvar = 1) = \Pr (\predvar = 0 | \truthvar = 1, \attrvar = 0)$ & 
\\ &
$\fairmat_\textsc{efnr} = \frac{1}{\cntvar} \begin{pmatrix}
0 & M_0 & 0 & 0 & 0 &-M_1 & 0 &0
\end{pmatrix}$ &
\EFNRpic
\\

Calibration within groups (CG)\cite{kleinberg2016inherent} &
$\Prob(y=1 | P_\theta(\mathbf{x}) = s, \attrvar = 1) = \Prob(y=1 | P_\theta(\mathbf{x}) = s, \attrvar = 0) = s$ &
\\ &
$\fairmat_\textsc{cg} = \CGmat$ &
\CGpic
\\

Positive class balance (PCB)\cite{kleinberg2016inherent} &
$\mathbb{E}(P_\theta | y=1 , \attrvar = 1) = \mathbb{E}(P_\theta | y=1 , \attrvar = 0)$ &
\\ &
$\fairmat_\textsc{pcb} =
\min_\fairvar(\posvar_\fairvar)
\begin{pmatrix}
\frac{v_1}{\posvar_1} & \frac{v_0}{\posvar_1} & 0 & 0 & -\frac{v_1}{\posvar_0} & -\frac{v_0}{\posvar_0} & 0 & 0 \end{pmatrix} $ &
\PCBpic
\\

Negative class balance (NCB)\cite{kleinberg2016inherent} &
$\mathbb{E}(P_\theta | y=0 , \attrvar = 1) = \mathbb{E}(P_\theta | y=0 , \attrvar = 0)$ &
\\ &
$\fairmat_\textsc{ncb} = \min_\fairvar(\cntvar_\fairvar - \posvar_\fairvar) \begin{pmatrix}
0 & 0 & \frac{v_1}{\cntvar_1 - \posvar_1} & \frac{v_0}{\cntvar_1 - \posvar_1} & 0 & 0 & -\frac{v_1}{\cntvar_0 - \posvar_0} & -\frac{v_0}{\cntvar_0 - \posvar_0}
\end{pmatrix}$ &
\NCBpic
\\

Relaxed Equalized Odds (REod)\cite{pleiss2017fairness} & 
$\alpha_0 FPR_0 + \beta_0 FNR_0 = \alpha_1 FPR_1 + \beta_1 FNR_1$ & 
\\ & 
$\fairmat_\textsc{REOd} = \begin{pmatrix}
0 & \frac{\beta_1}{\posvar_1} & \frac{\alpha_1}{\cntvar_1 - \posvar_1} & 0 & 0 & \frac{\beta_0}{\posvar_0} & \frac{\alpha_0}{\cntvar_0 - \posvar_0} & 0
\end{pmatrix} / N$
& \REOdPic
\\

\midrule

Predictive parity (PP)\cite{chouldechova2017fair} &
$\Prob(y = 1 | \hat{y}=1 , \attrvar = 1) = \Prob(y = 1 | \hat{y}=1 , \attrvar = 0)$ &
\\ &
$\frac 1 2 \fctvar^T \fairqfmat_\textsc{pp} \fctvar = (TP_1 FP_0 - TP_0 FP_1)/N^2$ &
\PPpic
\\

Equal false omission rate (EFOR) \tnote{1} &
$\Prob(y = 1 | \hat{y}=0 , \attrvar = 1) = \Prob(y = 1 | \hat{y}=0 , \attrvar = 0)$ &
\\ &
$\frac 1 2 \fctvar^T \fairqfmat_\textsc{efor} \fctvar = (TN_1 FN_0 - TN_0 FN_1)/N^2$ &
\EFORpic
\\

Conditional accuracy equality (CA)\cite{berk2018fairness} &
PP $\land$ EFOR &
\PPpic $\land$ \EFORpic
\\

\bottomrule
\end{tabular}
\begin{tablenotes}
\item[1]To our knowledge, EFOR has not been described in literature in isolation, but is used in the definition of conditional accuracy equality (CA)\cite{berk2018fairness}.
\item[2]Defined implicitly in \cite{chouldechova2017fair}.
\end{tablenotes}
\end{threeparttable}
}
\caption{
Some common group fairness definitions and corresponding abbreviations used throughout the paper in terms of linear functions $\fairfn(\fctvar) = \fairmat \fctvar$
or quadratic functions $\fairfn(\fctvar) = \frac 1 2 \fctvar^T \fairqfmat \fctvar$ that appear in the \PFOP{} \eqref{eqn:PFOP}.
There are two groups separated by the horizontal line:
those that are specified by linear functions (above),
or quadratic functions (below).
The graphical notation is described in \Cref{sec:linear}.
$P_\theta$ is the probability produced by a model (parameterized by $\theta$) of $\hat y =1$.
The fairness functions $\fairfn$ are uniquely defined only up to a normalization factor and overall sign.
}%
\label{tab:def_fairness}
\end{table*}

Our key insight is that the elements of the \fct{} encode all the information needed to study many notions of performance and group fairness.
The \fct{} is simply the stack of confusion matrices for each protected attribute $\fairvar$, as shown in \Cref{tab:fct}.
We focus on the simplest case, with one binary protected attribute $\fairvar\in\twoval$, and a binary classifier $\predvar\in\twoval$
for a binary prediction label $\truthvar\in\twoval$.\footnote{The arguments generalize to multiple and non-binary protected attributes with high-dimensional tensors.}

\begin{table}[!ht]
\small
    \centering
    \begin{tabular}{|c||c|c|}
    \toprule
    $a=1$ & $y=1$ & $y=0$\tabularnewline
    \midrule
    $\hat{y}=1$ &TP$_{1}$ & FP$_{1}$\tabularnewline
    \midrule
    \midrule
    $\hat{y}=0$ &FN$_{1}$ & TN$_{1}$\tabularnewline
    \bottomrule
    \end{tabular} \quad
    \begin{tabular}{|c||c|c|}
    \toprule
    $a=0$ & $y=1$ & $y=0$\tabularnewline
    \midrule
    $\hat{y}=1$ & TP$_{0}$ & FP$_{0}$\tabularnewline
    \midrule
    \midrule
    $\hat{y}=0$ & FN$_{0}$ & TN$_{0}$\tabularnewline
    \bottomrule
    \end{tabular}
    \caption{The \fct{},
    showing the two planes corresponding to the confusion matrix for each of the favored ($\fairvar = 1$)
    and disfavored groups ($\fairvar = 0$).}%
    \label{tab:fct}
\end{table}

Let us denote the elements of the fairness--confusion tensor as $TP_a, FP_a, FN_a, TN_a$, each element with subscripts indicating $\fairvar$, $\cntvar$ be the number of data points,
$\cntvar_\fairvar = TP_\fairvar + FN_\fairvar + FP_\fairvar + TN_\fairvar$
be the number of data points in each group $\fairvar \in \twoval$, and
$\posvar_\fairvar = TP_\fairvar + FN_\fairvar$
be the number of positive-class instances ($\truthvar = 1$) for each group.
Assume $\cntvar$, $\cntvar_\fairvar$ and $\posvar_\fairvar$ are known constants.
Unraveling the \fct{} into an 8-dimensional vector, we write it as
\begin{equation*}
\fctvar = {(TP_1, F\cntvar_1, FP_1, T\cntvar_1, TP_0, F\cntvar_0, FP_0, T\cntvar_0)}^T/N,
\end{equation*}
normalized and constrained to lie on $\simplex = \{\fctvar \ge 0 : \marginalmat \fctvar\ = \marginalvec,  \|\fctvar\|_1 = 1\}$, where $\marginalmat$ and $\marginalvec$ encode marginal sum constraints of the dataset (e.g., $TP_a + FN_a = M_a$) in matrix notations:
\begin{align*}
\small
\marginalmat  &= \begin{pmatrix}
    1 & 1 & 1 & 1 & 0 & 0 & 0 & 0 \\
    1 & 1 & 0 & 0 & 0 & 0 & 0 & 0 \\
    0 & 0 & 0 & 0 & 1 & 1 & 1 & 1 \\
    0 & 0 & 0 & 0 & 1 & 1 & 0 & 0
    \end{pmatrix}, \\ 
\marginalvec &= {(\cntvar_1, \posvar_1, \cntvar_0, \posvar_0)}^T / N. 
\end{align*}

We show below that some typical notions of group fairness can be reformulated as simple functions of $\fctvar$, namely as a form of $\phi(\fctvar) = 0$. 

\textbf{Demographic parity (DP)} states that each protected group should receive positive prediction at an equal rate: $\Prob(\hat{y} = 1 | \attrvar = 1) = \Prob(\hat{y} = 1  | \attrvar = 0)$, which is equivalent to $(TP_1 + FP_1)/{\cntvar_1} = (TP_0 + FP_0)/{\cntvar_0}$, or also the linear system $\phi(\fctvar) = \fairmat_\textsc{dp} \fctvar = 0$,
where
\begin{equation}
\fairmat_\textsc{dp} = \begin{pmatrix}
\cntvar_0 & 0 & \cntvar_0 & 0 & -\cntvar_1 & 0 & -\cntvar_1 & 0
\end{pmatrix} / \cntvar.
\label{eqn:dp}
\end{equation}
The choice of normalization, $1/N$, ensures that the matrix coefficients are in $[0, 1]$. We will refer to these matrices $\fairmat$ that encode information about the fairness conditions as fairness matrices. 

\textbf{Predictive parity (PP)} \cite{chouldechova2017fair}
states that the likelihood of being in the positive class given the positive prediction is the same for each group: $\Prob(y = 1 | \hat{y}=1 , \attrvar = 1) = \Prob(y = 1 | \hat{y}=1 , \attrvar = 0)$, which is equivalent to
$\frac{TP_1}{TP_1 + FP_1} =
    \frac{TP_0}{TP_0 + FP_0} \Longleftrightarrow \frac{TP_1}{TP_0} =
    \frac{FP_1}{FP_0}$. Unlike for DP, the marginal sum constraints do not relate $TP_a$ and $FP_a$, so this notion of fairness is \textit{not} linear in the \fct.
PP actually can be expressed using a \emph{quadratic} form:
\begin{equation}
\small
    \phi(\fctvar) = \frac 1 2 \fctvar^T \fairqfmat_\textsc{PP} \fctvar = 0,\quad
    \fairqfmat_\textsc{PP} = \begin{pmatrix}
    0 & 0 & 0 & 0 & 0 & 0 & -1 & 0 \\
    0 & 0 & 0 & 0 & 0 & 0 & 0 & 0 \\
    0 & 0 & 0 & 0 & 1 & 0 & 0 & 0 \\
    0 & 0 & 0 & 0 & 0 & 0 & 0 & 0 \\
    0 & 0 & 1 & 0 & 0 & 0 & 0 & 0 \\
    0 & 0 & 0 & 0 & 0 & 0 & 0 & 0 \\
    -1 & 0 & 0 & 0 & 0 & 0 & 0 & 0 \\
    0 & 0 & 0 & 0 & 0 & 0 & 0 & 0
    \end{pmatrix}.
\end{equation}

\textbf{Calibration within groups (CG)}~\cite{kleinberg2016inherent},
when specialized to binary classifiers and binary protected classes,
can be written as the system of equations $FN_a = v_0 (FN_a + TN_a); TP_a  = v_1 (TP_a + FP_a)$, 
where the $v_i$s are scores satisfying $0 \le v_0 < v_1 \le 1$
and have no implicit dependence on any entries of the \fct.
We can rewrite this this condition explicitly as the matrix equation
$\phi(\fctvar) = \fairmat_{\textsc{cg}} \fctvar = 0$ with a fairness matrix
\begin{equation}
\label{eqn:cg}
\small
\fairmat_{\textsc{cg}} = \CGmat.
\end{equation}

\textbf{Equalized odds (EOd)} \cite{hardt2016equality} states that true-positive rates and false-positive rates are the same for both groups, which can be expressed as a linear system $\phi(\fctvar) = \fairmat_{\textsc{EOd}} \fctvar = 0$ with a fairness matrix
\begin{equation}
\fairmat_{\textsc{EOd}} = \frac{1}{N} \begin{pmatrix}
\posvar_0 & 0 & 0 & 0 & -\posvar_1 & 0 & 0 & 0 \\
0 & 0 & \cntvar_0 - \posvar_0 & 0 & 0& 0 & -\cntvar_1 + \posvar_1&0
\end{pmatrix}
\end{equation}
where each row respectively corresponds to conditions for Equality of Opportunity (EOp)~\cite{hardt2016equality} and Predictive Equality (PE)~\cite{chouldechova2017fair}.
Likewise, vertically stacking multiple fairness matrices results in a fairness matrix corresponding to the conjunction of different fairness notions.

In \Cref{tab:def_fairness} we generalize this formulation to a wide majority of group fairness definitions in the literature, along with their abbreviations used throughout the paper. We find that most of the definitions take either linear or quadratic form with respect to $\fctvar$. We further introduce a graphical notation to help visualize which components of the \fct{}
participate in the fairness definition.
Depict the \fct{} as
\begin{tikzpicture}
\draw[help lines, step=4pt] (0,0) grid (8pt, 8pt);     
\draw[help lines, step=4pt] (12pt,0) grid (20pt, 8pt); 
\end{tikzpicture}
,
with the left matrix for the favored class ($\fairvar = 1$)
and the right matrix for the disfavored class ($\fairvar = 0$).
Since each component of $\fctvar$ corresponds to some element of the \fct,
we shade each component that appears in the equation.
Blue shading denotes the favored class,
while red shading denotes the disfavored class.
We further distinguish two kinds of dependencies.
Components that have a nonzero coefficient in the matrix are shaded fully.
However, the values of these coefficients themselves can depend on other components,
albeit implicitly, and we shade these implicit components in a lighter shade. Putting this all together,
we can represent DP in \eqref{eqn:dp} graphically as
\DPpic, EOd as \EOppic$\land$\PEpic, PP as \PPpic, with the superscript denoting the quadratic order of the term. As shown in the third column of \Cref{tab:def_fairness}, all group fairness notions can be effectively described in this notation.


\section{Optimization over the Fairness--confusion Tensor}%
\label{sec:optimization}

The fairness--confusion tensor $\fctvar$ allows for a succinct linear and quadratic characterization of group fairness definitions in the literature. We naturally consider the following family of optimization problems over $\fctvar \in \simplex$, where the objective function is constructed so that the solution reflects trade-offs between fairness and performance. 

\begin{definition}
Let $\perffn^{(i)} : \simplex \rightarrow [0, 1]$ be performance metrics (indexed by $i$) with best performance 0 and worst performance 1,
$\fairfn^{(j)}(\fctvar)$ be fairness functions (indexed by $j$) with
$\secondlmvar_i$, $\lmvar_j$ be real constants with $\secondlmvar_0 = 1$. 
Then, the \emph{\PFOP{} (PFOP)} is a class of optimization problem of form:
\begin{equation}
\argmin_{\fctvar \in \simplex}
\sum_{i\ge0} \secondlmvar_i \perffn^{(i)}(\fctvar)
+ \sum_{j\ge0} \lmvar_j \fairfn^{(j)}(\fctvar)
\label{eqn:PFOP}
\end{equation}
\end{definition}

PFOP is a general optimization problem containing two groups of terms;
the first quantifying performance loss;
the second quantifying unfairness.
The restriction $\fctvar\in\simplex$ is necessary to ensure that $\fctvar$ is a valid \fct{} that obeys the requisite marginal sums. In our discussion below, it will be convenient to consider solutions with explicit bounds on their optimality.

\begin{definition}%
\label{def:approxsol}
Let $\fairepsvar \ge 0 $ and $\perfepsvar \ge 0$.
Then, a \emph{$(\fairepsvar, \perfepsvar)$-solution to the PFOP}
is a $\fctvar$ that satisfies \eqref{eqn:PFOP}
such that $\sum_j \lambda_j \fairfn^{(j)}(\fctvar) \le \fairepsvar$
and $\sum_i \mu_i \perffn^{(i)}(\fctvar) \le \perfepsvar$.
\end{definition}

The parameters $\fairepsvar$ and $\perfepsvar$ represent the sum total of deviation from perfect fairness and perfect predictive performance respectively. Unless otherwise stated, the rest of the paper is dedicated to analyzing one of the simplest instantiations of PFOP, defined below.

\begin{definition}
The \emph{\LAFOP{} (LAFOP)} is a PFOP with accuracy (or classification error rate) as the performance function $\perffn^{(0)}$, and $K\ge1$ fairness constraints in the form of a fairness matrix $\fairmat$ (each row indexed by $j$), with
\begin{equation}
\begin{aligned}
    \fairfn^{(j)}(\fctvar) &= (\fairmat_{j,*} \fctvar)^2, \quad j = 0, ..., K-1\\
    \perffn^{(0)}(\fctvar) &= (\perfvec \cdot \fctvar)^2, \\
    \perfvec &= {(0, 1, 1, 0, 0, 1, 1, 0)}^T, \\
    \lmvar &= \lmvar_0 = ... = \lmvar_{K-1}.
\end{aligned}
\end{equation}
In other words, LAFOP is the problem
\begin{equation}
\argmin_{\fctvar \in \simplex}
{(\perfvec \cdot \fctvar)}^2
+ \lmvar \Vert \fairmat \fctvar \Vert^2_2,
\label{eqn:LAFOP}
\end{equation}
\end{definition}
where $\perfvec \cdot \fctvar$  encodes the usual notion of classification error, and $\fairmat$ encodes $K$ linear fairness functions stacked together as the regularizer.
A single hyperparameter $\lmvar$ specifies the relative importance of satisfying the fairness constraints while optimizing classification performance, with $\lmvar = 0$ considering only performance and disabling all fairness constraints, and $\lmvar = \infty$ imposing fairness constraints without regard to accuracy.

LAFOP is a convex optimization problem which is simple to analyze.
Despite its simplicity, LAFOP encompasses many situations involving linear notions of fairness,
allowing us to reason about multiple fairness constraints as well as fairness--accuracy trade-offs under versatile scenarios.

\subsection{Reduction to a post-processing method for fair classification}

PFOP and LAFOP do not assume anything about the model, therefore are designed to be model-agnostic. In this section we highlight the versatility of LAFOP by showing that adding a model-specific constraint on LAFOP reduces it to a post-processing algorithm for fair classification.

Post-processing method, in particular for EOd as introduced in \citet{hardt2016equality}, solves the following optimization problem for $\tilde{Y}$, which is a post-processed, supposedly fair, classifier, given $\hat{Y}$, a vanilla classifier:
\begin{multline}
    \min_{\tilde{Y}} \mathbb{E}l(\tilde{Y}, Y) \text{ such that } \gamma_0(\tilde{Y}) = \gamma_1(\tilde{Y}) \\ 
    \text{ and } \gamma_0(\tilde{Y}) \in P_0(\hat{Y}), \gamma_1(\tilde{Y}) \in P_1(\hat{Y})
    \label{eqn:hardt}
\end{multline} where $\gamma_a(\tilde{Y})$ represents EOd constraints for $\tilde{Y}$ as a tuple of ($FPR_a$, $TPR_a$), and $P_a(\hat{Y})$ is a model-specific set of feasible $\gamma_a$ values, defined as $P_a(\hat{Y}) = \text{convhull}\{(0,0), \gamma_a(\hat{Y}), \gamma_a(1 - \hat{Y}), (1,1)\}$. 
All the components of \eqref{eqn:hardt} can be rewritten in terms of $\hat{\fctvar}$ and $\tilde{\fctvar}$, the fairness--confusion tensors corresponding to the classifiers $\hat{Y}$ and $\tilde{Y}$ respectively. This yields a LAFOP over $\tilde{z}$ with additional model-specific constraints derived from $\hat{z}$ on the solution space. More formally, we have the following optimization problem for post-processing:
\begin{definition}
Given a classifier to be post-processed and its corresponding \fct{} $\hat{z}$, the \emph{model-specific LAFOP} (MS-LAFOP) for EOd is the variant of LAFOP with model-specific constraints on the solution space as the following:
\begin{equation}
\argmin_{\tilde{\fctvar} \in \hat{\mathcal{K}}}
{(\perfvec \cdot \tilde{\fctvar})}^2
+ \lmvar \Vert \fairmat_{\textsc{EOd}} \tilde{\fctvar} \Vert^2_2, \text{ where }
\label{eqn:ms-lafop}
\end{equation} where
\begin{multline*}
\hat{\mathcal{K}} = \big\{ \tilde{\fctvar} \ge 0 : \marginalmat \tilde{\fctvar}\ = \marginalvec,  \|\tilde{\fctvar}\|_1 = 1, \\ 
\beta_a(\tilde{z}) \in \text{convhull}\left\{ (0,0), \beta_a(\hat{z}), \beta_a(1-\hat{z}), (1,1)\right\} \forall a \big\}
\end{multline*} with $\beta_a$ expressing ($FPR_a$, $TPR_a$) tuples computed from the corresponding \fct{} of group $a$.
\end{definition}
From the solution of MS-LAFOP, it is possible to compute mixing rates for post-processing the given classifier. We note that MS-LAFOP can be extended to other group fairness notions as long as the model-specific constraints are accordingly set up for them. For more details, refer to \Cref{sec:post-process-apdx}.


\begin{table*}[ht]
\small
\centering
\begin{tabular}{lc}
Sets of fairness definitions & Necessary conditions \\\hline \hline
\{CG, PP, DP, and any of EOp, PE, PCB, NCB, EFOR\} & $M_0 = M_1$ and $N_0 = N_1$ \\ \hline
\{CG, DP, and any of EOp, PE, PCB, NCB, EFOR\} & EBR only \\\hline
\multirow{2}{40em}{\{CG,EOp\}, \{CG,PCB\}, \{CG,EOp,PCB\},\{CG,EFOR,EOp\}, \{CG,EFOR,PCB\},\{CG,EFOR,EOp,PCB\}} &  $v_{0}=0$ \\
& or EBR \\ \hline
\multirow{2}{40em}{\{CG,PE\}, \{CG,NCB\}, \{CG,EOp,NCB\}, \{CG,EFOR,PE\}, \{CG,EFOR,NCB\}, \{CG,EFOR,EOp,NCB\}} &  $v_{1}=1$ \\
& or EBR \\ \hline
\multirow{2}{40em}{\{CG,EOd\}\cite{pleiss2017fairness}, \{CG, PCB, NCB\} \cite{kleinberg2016inherent},\{CG,EOd,PCB,NCB\}, \{CG,EFOR,EOd\}, \{CG,EFOR,PCB,NCB\},\{CG,EFOR,EOd,PCB,NCB\}} & ($v_{0}=0$ and $v_{1}=1$) \\
& or EBR \\ \hline \hline
\end{tabular}
\caption{Some sets of fairness definitions containing Calibration(CG), which are incompatible in the sense of \Cref{def:incompat} (left-column), together with their necessary conditions to be compatible (right column).
EBR is the equal base rate condition, $M_0/N_0 = M_1/N_1$. For other abbreviations, refer to \Cref{tab:def_fairness}. These are all special cases of \Cref{thm:calib_incomp}, while not exhaustive.}
\label{tab:cg-impossible}
\end{table*}

\section{Incompatible Group Fairness Definitions}%
\label{sec:incompatible} 

In this section, we show how LAFOP yields a more general view of understanding group fairness incompatibility results. As $\lmvar \to \infty$, for linear fairness functions $\fairfn^{(i)}(\fctvar) = \fairmat^{(i)} \fctvar$, LAFOP becomes equivalent to solving the following linear system of equations:
\begin{equation}
\small
\begin{pmatrix}
\fairmat^{(0)}\\
\vdots\\
\fairmat^{(K-1)}\\
\marginalmat
\end{pmatrix} \fctvar =
\begin{pmatrix}
0\\
\vdots\\
0\\
\marginalvec
\end{pmatrix}, \, \fctvar\ge0,
\label{eqn:linear-fair-compat}
\end{equation} 
Notice the compatibility of fairness conditions encoded by these $K$ fairness matrices $\fairmat^{(i)}$ is equivalent to having infinitely many solutions to the above linear system. We formally define (in)compatibility of fairness notions below based on this observation.
 
\begin{definition}
Let $\fairset = {\{\fairfn^{(i)}\}}^{K-1}_{i=0}$ be a set of linear fairness functions, encoded in a fairness matrix $\fairmat$ (of which each row corresponds to $\fairfn^{(i)}$), and let $\rankvar$ be the number of solutions for the system in \eqref{eqn:linear-fair-compat}. If $\rankvar = 0$, then $\fairset$ is said to be incompatible. Otherwise, $\fairset$ is compatible. When $\fairset$ is incompatible, some additional set of constraints on the dataset or the model is required for it to be compatible. 
\label{def:incompat}
\end{definition}

This means that in general, incompatibility results among the group fairness definitions can be proven simply by asking if and when solutions exist to their corresponding linear system of form~\eqref{eqn:linear-fair-compat}.

\subsection{The incompatibility involving CG}%
\label{sec:kleinberg}

We introduce a general incompatibility result involving CG that leads to many other new results as well as the one from \citet{kleinberg2016inherent}.

\begin{theorem}%
\label{thm:calib_incomp}
Let $B=2$ be the number of bins in the definition of calibration within groups fairness (CG) \cite{kleinberg2016inherent},
and $v_0$, $v_1$ be the scores, with $0 \le v_0 < v_1 \le 1$,
and $K>1$ with $\fairfn^{(0)}(\fctvar) = \fairmat_\textsc{CG}\fctvar$.
Then, the corresponding \eqref{eqn:linear-fair-compat} has the only solution
\begin{equation}
z_0 = \frac{1}{\cntvar (v_1 - v_0)}
\begin{pmatrix}
v_1 ( \posvar_1 - \cntvar_1 v_0) \\
v_0 (-\posvar_1 + \cntvar_1 v_1) \\
(1 - v_1)( \posvar_1 - \cntvar_1 v_0)  \\
(1 - v_0)(-\posvar_1 + \cntvar_1 v_1) \\
v_1 ( \posvar_0 - \cntvar_0 v_0) \\
v_0 (-\posvar_0 + \cntvar_0 v_1) \\
(1 - v_1)(\posvar_0 - \cntvar_0 v_0) \\
(1 - v_0)(-\posvar_0 + \cntvar_0 v_1)
\end{pmatrix},
\label{eqn:impossible-sol}
\end{equation}
and only when
\begin{equation}
    0\le v_0 \le \min_a\left(\frac{\posvar_a}{\cntvar_a}\right)
     \le \max_a\left(\frac{\posvar_a}{\cntvar_a}\right) \le v_1 \le 1.
\label{eqn:impossible-sol-cond}
\end{equation}
Otherwise, no solution exists.
\end{theorem}

\Cref{thm:calib_incomp} yields other extended results regarding the incompatibility of CG and other notions of fairness. As one canonical instance, simply substituting $z_0$ in \eqref{eqn:impossible-sol} to the linear system of the form in \eqref{eqn:linear-fair-compat} with PCB and NCB fairness matrices yields the following corollary, which is equivalent to the result presented in \citet{kleinberg2016inherent} (proof is in \Cref{sec:calib-proof}). 

\begin{corollary}[Re-derivation of \cite{kleinberg2016inherent}]%
\label{thm:tr2}
Consider a classifier that satisfies CG, PCB and NCB fairness simultaneously.
Then, at least one of the following statements is true:
\begin{enumerate}[nolistsep]
    \item the data have equal base rates for each class $\fairvar$, i.e.\
    $\posvar_0/\cntvar_0 = \posvar_1/\cntvar_1$, or
    \item the classifier has perfect prediction, i.e.\  $v_0 = 0$ and $v_1 = 1$.
\end{enumerate}
\end{corollary}

Similar approach can be applied to derive incompatibilities of CG with other linear and quadratic notions of fairness as below (proofs in \Cref{sec:cgdp}, \Cref{sec:cgpp}).

\begin{corollary}(Linear notion of fairness: DP)
\label{thm:cgdp}
Consider a classifier that satisfies CG and DP fairness simultaneously.
Then, the data have equal base rates for each group $\fairvar$.
\end{corollary}

\begin{corollary}(Quadratic notion of fairness: PP)
\label{thm:cgpp}
Consider a classifier that satisfies CG and PP fairness simultaneously.
Then, at least one of the following is true:
\begin{enumerate}[nolistsep]
\item 
$v_0 = (M_1 - M_0) / (N_1 - N_0)$.
\item 
$v_1 = 1$.
\end{enumerate}
\end{corollary}

From \Cref{thm:calib_incomp} and its corollaries, we curate the extended incompatibility results involving CG in \Cref{tab:cg-impossible} along with conditions for compatibility. To our knowledge, all cases other than the bottom row of the table are new. 

\subsection{The incompatibility of \{PE, EFNR, PP\}}

Using the same logic as the previous section, we re-derive an incompatibility result in \citet{chouldechova2017fair} and provide more precise necessary conditions for compatibility. For details of the proof, refer to \Cref{sec:proof-chould}.

\begin{theorem}[Restatement of \citet{chouldechova2017fair}]
\label{thm:chouldechova}
Consider a classifier that satisfies \{PE, EFNR, PP\}. Then, at least one of these statements must be true:
\begin{enumerate}[nolistsep]
    \item The classifier has no true positives.
    \item The classifier has no false positives.
    \item Each protected class has the same base rate.
\end{enumerate}
\end{theorem}

\Cref{thm:chouldechova} systematically shows that equal false positive rates, equal false negative rates, and predictive parity are compatible only under specific data/model-dependent circumstances, that were otherwise not clear in the original statements in \citet{chouldechova2017fair}. 

\section{Experiments}%
\label{sec:experiments}

In this section we show how the FACT diagnostic can practically show the relative impact of several notions of fairness on accuracy on synthetic and real datasets\footnote{Code available:  \href{https://github.com/wnstlr/FACT}{\texttt{github.com/wnstlr/FACT}}}. First we introduce FACT Pareto frontiers which characterize a model's achievable accuracy for a given set of fairness conditions, as a tool for understanding the trade-offs and contextualizing some recent works in fair classification (\Cref{sec:frontier}). We then explore a model-agnostic assessment of multiple fairness conditions via LAFOP (\Cref{sec:exact-relaxed}, \Cref{sec:multi}), as well as a model-specific assessment of post-processing methods in fair classification via MS-LAFOP (\Cref{sec:post-exp}, \Cref{sec:post-process-apdx}). 

\subsection{Datasets}%
\label{sec:datasets}
We study a synthetic dataset similar to that in \citet{zafar2015fairness},
consisting of two-dimensional features along with a single binary protected attribute that is either sampled from an independent Bernoulli distribution (``unbiased'' variant, denoted \textbf{S(U)}),
or sampled dependent on the features (``biased'' variant, denoted \textbf{S(B)}). The synthetic dataset consists of two-dimensional data $\data = (x_0, x_1)$ that follow the Gaussian distributions
\begin{equation}
\small
\begin{aligned}
\data | \truthvar = 1 \sim & \gaussian\left(\begin{pmatrix}
2\\2
\end{pmatrix}, \begin{pmatrix}
5 & 1 \\ 1 & 5
\end{pmatrix} \right) \\ 
\data | \truthvar = 0 \sim & \gaussian\left(\begin{pmatrix}
-2\\-2
\end{pmatrix}, \begin{pmatrix}
10 & 1 \\ 1 & 3
\end{pmatrix} \right).
\end{aligned}
\end{equation} 

For the S(U) dataset, the protected attribute value is independent of $\data$ and $\truthvar$,
and is instead distributed according to the Bernoulli distribution $\fairvar \sim \bernoulli\left(\frac 1 2\right)$. This notion of fairness was described in \cite{calders2009indep}.

For the S(B) dataset, the protected attribute value is assigned as $a|\data = \sgn(x_0)$, which corresponds to a situation when some features (but not all)
encode a protected attribute. 

We also study the
UCI Adult dataset~\cite{UCIMLrepo},
a census dataset used for income classification tasks where we consider sex as the protected attribute of interest.

\subsection{FACT Pareto frontiers}
\label{sec:frontier}

\begin{figure}[ht]
\centering
\includegraphics[width=\columnwidth]{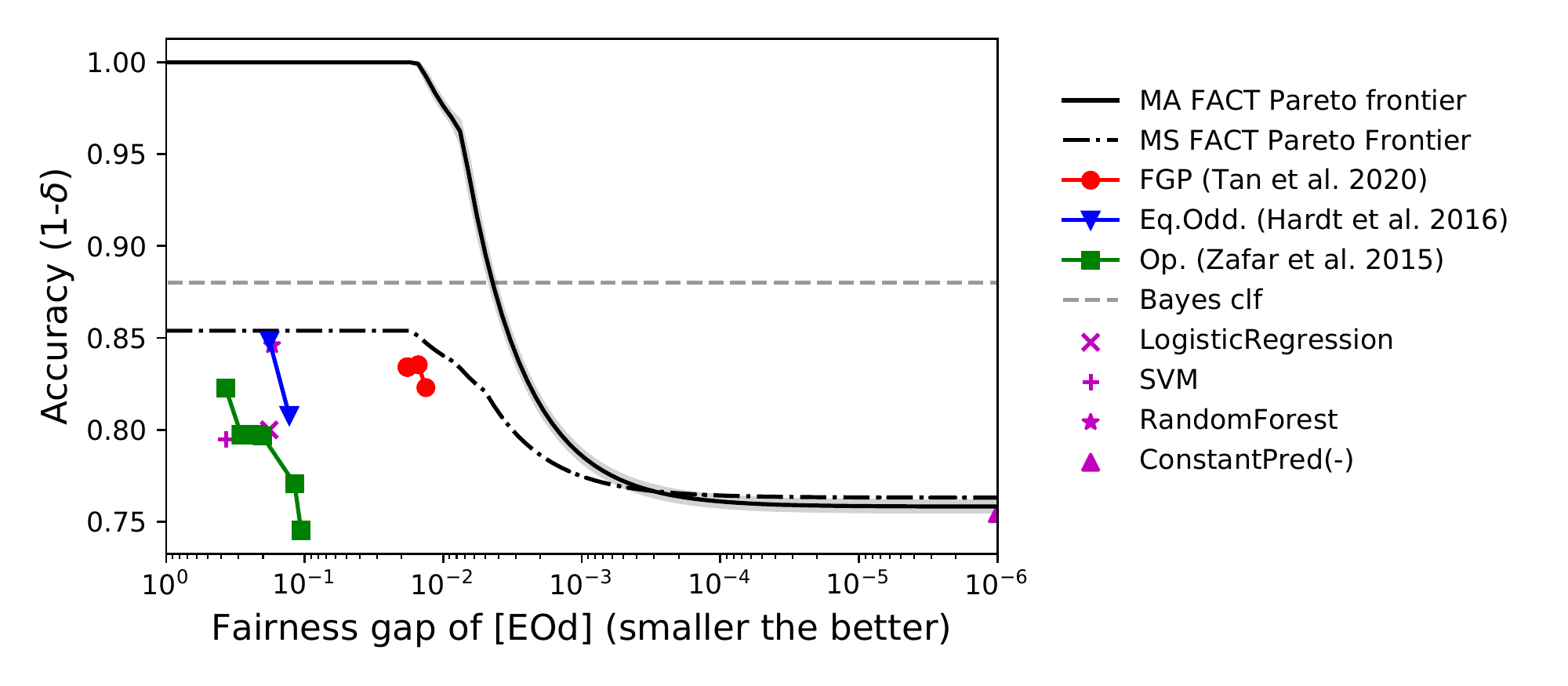}
\caption{Model-agnostic (MA) and model-specific (MS) FACT Pareto frontiers of equalized odds on the Adult dataset. Three fair models (FGP, Eq.Odd., Op.) are shown in context by varying the strength of the fairness condition imposed, along with some baseline models (LR, SVM, RF, ConstantPrediction). The MA frontier should be interpreted relative to the Bayes error because it is oblivious to it --- $\delta=0$ means that the upper bound of the accuracy is the accuracy of the Bayes classifier, not 1. The MS frontier on the other hand provides realistic more bounds.}
\label{fig:frontier}
\end{figure}

With LAFOP and MS-LAFOP, one can naturally consider a FACT Pareto frontier of accuracy and fairness by plotting $(\fairepsvar, \perfepsvar)$ values of the $(\fairepsvar, \perfepsvar)$-solutions. In this section, we want to highlight the use of this frontier in the context of several published results in the literature as well as its implications.

The FACT Pareto frontier can be computed both in model-agnostic (MA) and model-specific (MS) scenarios by solving LAFOP and MS-LAFOP respectively, and \Cref{fig:frontier} shows such example on the Adult dataset for EOd fairness. We also consider three fair classification models: \textbf{FGP}~\cite{tan2019learning}, \textbf{Op.}~\cite{zafar2015fairness}, and \textbf{Eq.Odd.}~\cite{hardt2016equality}, individually representing three different approaches one can take in training fair models (imposing fairness before, during, or after training). Some baseline models (logistic regression, SVM, random forest) are also plotted for reference, and a perfectly fair classifier (ConstantPredict: predicting all instances to be negative) on the bottom right corner is considered as an edge-case. 

It is important to note that the MA FACT Pareto frontier should be interpreted as characterizing the model's achievable accuracy \emph{relative} to the Bayes error (i.e., the degree to which the added fairness constraints adversely impact the Bayes error), which in this case is empirically estimated at around 0.12 from a wide range of ML models that have been tested on the Adult datset~\cite{chakrabarty2018statistical}. This relatively less realizable bound calls for a model-specific counterpart, the MS FACT Pareto frontier, which limits the frontier to be derived from a given pre-trained classifier. As shown in \Cref{fig:frontier}, it indeed provides a more reasonable frontier for the models considered.

Placing different types of classifiers on the frontier, it is easy to visually grasp strengths and weaknesses of each models. FGP seems to outperform all other models in terms of the trade-off, while Op and EqOdd suffer more from early accuracy drops. The frontier further informs that for any model trained, only for fairness gaps below $10^{-2}$ will the accuracy start to suffer. Such understanding of the trade-offs will be helpful in anticipating practical limitations of models to be trained, as well as in comparing multiple models to determine which is better-suited for different situations. 

\begin{figure*}[ht]
\includegraphics[width=\textwidth]{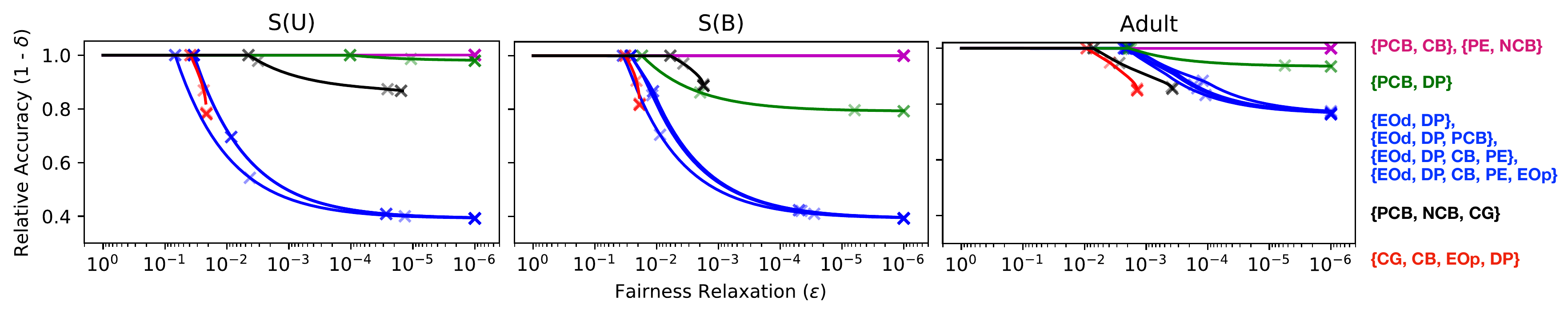}
\caption{Model-agnostic FACT Pareto frontier for different groups of fairness notions (colored and grouped according to their convergence value as $\fairepsvar \to  0$) for three datasets (\Cref{sec:datasets}). The bottom two groups of fairness notions are incompatible (black, red), hence the halted trajectories before reaching smaller values of $\fairepsvar$. Similar convergence behaviors within the fairness groups in blue reflect the dominance of \{EOd, DP\} -- any additional fairness notions added on top of these have no impact on the convergence value. Best viewed in color.}
\label{fig:eps-delta-curve}
\end{figure*}

In the rest of the following sections and figures, for the model-agnostic analysis, $\perfepsvar$ should be interpreted in reference to the Bayes error, i.e $\perfepsvar = 0$ means that the upper bound of the best-achievable accuracy is the accuracy of the Bayes classifier, not 1.

\subsection{Model-agnostic scenario with multiple fairness conditions}
\label{sec:exact-relaxed}

We are now interested in how a \emph{group} of fairness conditions simultaneously affect accuracy. This can be assessed by looking at the shape of the MA FACT Pareto frontier of LAFOP with multiple fairness constraints, particularly $\perfepsvar$ values of $(\fairepsvar, \perfepsvar)$-solutions when $\fairepsvar$ is varied to be zero (or very close to it) on multiple fairness notions. \Cref{fig:eps-delta-curve} shows this in two different ways: (i) ($\fairepsvar$,$\perfepsvar$)-solutions obtained when fairness conditions are imposed as hard inequality constraints instead of as regularizers, i.e. solving $\argmin_{\fctvar \in \simplex}
{(\perfvec \cdot \fctvar)^2} \text{ s.t. } \Vert \fairmat \fctvar \Vert^2_2 \leq \fairepsvar$ (solid line), and (ii) 
($\fairepsvar$,$\perfepsvar$)-solutions obtained from the LAFOP \eqref{eqn:LAFOP} while varying $\lmvar$s (crosses). Different groups of fairness notions are colored according to their convergence behaviors. 

Similar trajectories and convergence of the curves allow us to identify fairness notions that come ``for free'' given some others, in terms of additional accuracy drops. In other words, the Pareto frontiers are effective at demonstrating the relative strength of the fairness notions within a group. For instance, under \{EOd, DP\} (third group, blue) the best attainable accuracy drops by over 60 percent for S(U) and S(B), but we also observe that adding CB, PE, and/or PCB on top of them causes no additional accuracy drop -- \{EOd, DP\} essentially determines $\perfepsvar$ for the entire group of fairness notions in blue.

The MA FACT Pareto frontiers for multiple fairness conditions also show not only the existing incompatibility of the fairness notions, but also how much relaxation is required for them to be approximately compatible. 
The halted trajectories before hitting much smaller $\fairepsvar$ for the bottom two groups in black and red clearly verify this. Because the S(U) dataset has a smaller base rate gap between the groups compared to the Adult or the S(B) dataset by design, the incompatibility in S(U) becomes only visible at a much smaller $\fairepsvar$ value. 

Taking a more macroscopic perspective, the MA FACT Pareto frontiers also show which dataset allows overall better trade-off scheme compared to the others. Because the S(U) dataset was designed to be less biased compared to the S(B) dataset, it exhibits significantly smaller drop in overall accuracy, particularly for the green group involving DP. The way S(U) was designed aligns with this observation, as the sensitive attributes were randomly sampled independently from the features. However, EOd and DP together (in blue) drives down the accuracy just like the biased counterpart, which demonstrates how conservative EOd fairness is for these datasets. 

More observations and experiments are presented in \Cref{sec:multi}. It is possible to further extend these analyses to an arbitrary number of fairness constraints imposed on LAFOP, as well as to other performance metrics like precision or recall as seem fit. 

\subsection{Model-specific scenario with post-processing methods}
\label{sec:post-exp}

While the MA FACT Pareto frontier shows a broader trade-off landscape for any classifiers, model-specific analysis using MS-LAFOP in \eqref{eqn:ms-lafop} can be helpful in practice with more reasonable MS Pareto frontiers. Also after solving the MS-LAFOP, its solution can be used to compute the mixing rates for post-processing any given classifier just like done in \citet{hardt2016equality}. For more details, refer to \Cref{sec:post-process-apdx}.

\begin{figure}[ht]
\centering
\includegraphics[width=0.8\columnwidth]{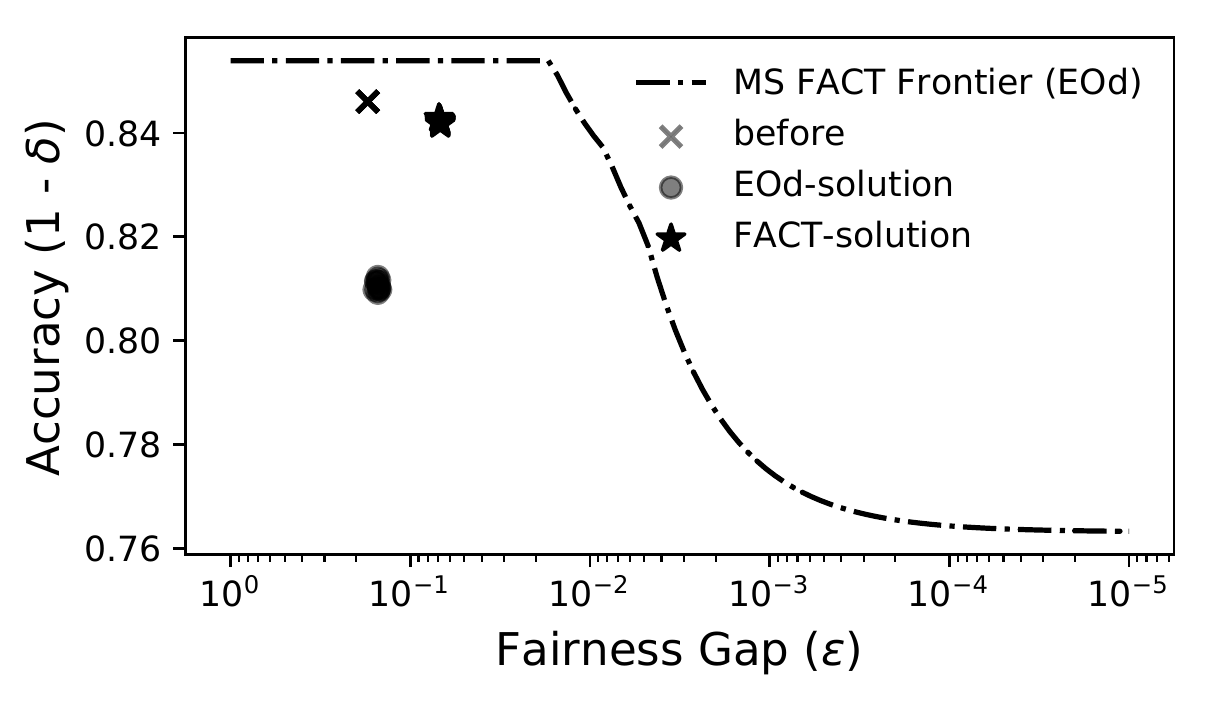}
\caption{Model-specific FACT Pareto frontier of EOd on Adult dataset. Compared to the model-agnostic frontier, it yields a more realizable bounds on the trade-off between fairness and accuracy. Post-processed solutions for the given classifiers (crosses) using the algorithm in \cite{hardt2016equality} (circles, EOd-solution) and FACT (stars, FACT-solution) are also shown. The FACT-solutions suffer significantly less from the trade-off, yielding competitive accuracy to the original classifiers while achieving smaller fairness gaps compared to the EOd-solutions.}
\label{fig:post-fig}
\end{figure}

\Cref{fig:post-fig} shows the MS FACT Pareto frontier of EOd computed from MS-LAFOP for the Adult dataset (it is a zoomed-in version of the MA FACT Pareto frontier in \Cref{fig:frontier}). We also plot two types of post-processed classifiers: EOd-solutions using the algorithm in \citet{hardt2016equality} (circles), and FACT-solutions using MS-LAFOP (stars). EOd solutions undergo steeper trade-off while the FACT-solutions are able to find a better configuration with smaller fairness gaps, retaining a competitive accuracy level to the original classifier (cross). 

\section{Conclusions}

The \name\ diagnostic facilitates systematic reasoning about different kinds of trade-offs involving arbitrarily many notions of performance and group fairness notions, which all can be expressed as functions of the \fct. 
In our formalism, the majority of group fairness definitions in the literature are in fact linear or quadratic thus are easy to be imposed as constraints to the PFOP.
The \name{} diagnostic further benefits from elementary linear algebra and convex optimization to provide a unified perspective of viewing fairness--fairness trade-offs and fairness--performance trade-offs. We have also empirically demonstrated the practical use of the \name{} diagnostic in several scenarios. 
Many of the presented results require only linear fairness functions and accuracy, as in the LAFOP/MS-LAFOP setting. Nevertheless, it is easy to extend this to quadratic fairness functions with more varied performance metrics depending on different use cases.
We also briefly introduce a small theoretical result regarding fairness--accuracy trade-offs using the FACT diagnostic in \Cref{sec:cg-accuracy}, which deserves further analysis.

\section*{Acknowledgements}

We thank Valerie Chen, Jeremy Cohen, Amanda Coston, Mikhail Khodak, Jeffrey Li, Liam Li, Gregory Plumb, Nick Roberts, and Samuel Yeom for helpful feedback and discussions. This work was supported in part by DARPA FA875017C0141, the National Science Foundation grants IIS1705121 and IIS1838017, an Okawa Grant, a Google Faculty Award, an Amazon Web Services Award, a JP Morgan A.I. Research Faculty Award, and a Carnegie Bosch Institute Research Award. JSK acknowledges support from Kwanjeong Educational Fellowship. Any opinions, findings and conclusions or recommendations expressed in this material are those of the author(s) and do not necessarily reflect the views of DARPA, the National Science Foundation, or any other funding agency.

This paper was prepared for information purposes by the Artificial Intelligence Research  group of JPMorgan Chase \& Co and its affiliates (``JP Morgan''), and is not a product of the Research Department of JP Morgan. JP Morgan makes no representation and warranty whatsoever and disclaims all liability, for the completeness, accuracy or reliability of the information contained herein.  This document is not intended as investment research or investment advice, or a recommendation, offer or solicitation for the purchase or sale of any security, financial instrument, financial product or service, or to be used in any way for evaluating the merits of participating in any transaction, and shall not constitute a solicitation under any jurisdiction or to any person, if such solicitation under such jurisdiction or to such person would be unlawful. © 2020 JPMorgan Chase \& Co. All rights reserved.

\bibliography{main}
\bibliographystyle{icml2020}

\clearpage
\onecolumn
\appendix

\section{Proof of \Cref{thm:calib_incomp}}
\label{sec:kleinbergproof}

A useful strategy is to solve \eqref{eqn:linear-fair-compat} for a set of solutions,
then ask if any of these solutions satisfies an additional fairness constraint
$\fairfn^{(K)}(\fctvar) = 0$.
This proof, as well as many of the ones below, illustrate this strategy in practice.

\begin{proof}
First, set $K=1$ and $\fairmat^{(0)} = \fairmat_\textsc{cg}$ in \eqref{eqn:linear-fair-compat}.
Since $v_0 \ne v_1$, the matrix $\fairmat$ is full rank and therefore admits the solution \eqref{eqn:impossible-sol}. Considering $\fctvar_0 \ge 0 $ yields immediately the condition \eqref{eqn:impossible-sol-cond}.

Next, set $K>1$. Then either $\fctvar_0$ is a solution (which is the case when all other fairness notions are linear and linearly dependent on $\begin{pmatrix} \fairmat_\textsc{cg} \\ \marginalmat \end{pmatrix}$),
or otherwise no solution exists to both \eqref{eqn:linear-fair-compat} and
$\fairfn^{(1)}(\fctvar) = \cdots = \fairfn^{(K-1)}(\fctvar) = 0$ simultaneously.
\end{proof}

This theorem states that $\Phi=$\{CG\} is incompatible when $v_0 \ne v_1$, since it is a singleton set of incompatible fairness.

The condition $v_0 \ne v_1$ is necessary in \Cref{thm:calib_incomp},
which is reasonable to assume as we would expect the positive class to have a higher score than the negative class in the definition of CG.
We can prove the necessity of this condition by contradiction. In the degenerate case $v_0 = v_1 = v$, $\Phi=$\{CG\} is a set of compatible fairness notions. It turns out that \eqref{eqn:linear-fair-compat} with $K=1$ is only on rank 6.
Denoting \textcircled{$i$} as the $i$th row of the matrix, we have two linear dependencies,
$\text{\textcircled{5}} + \text{\textcircled{6}} + v \text{\textcircled{1}} = \text{\textcircled{2}}$ and
$\text{\textcircled{7}} + \text{\textcircled{8}} + v \text{\textcircled{3}} = \text{\textcircled{4}}$.
There is no longer a unique solution to the \eqref{eqn:linear-fair-compat}; instead, we have a two-parameter family of solutions,
\begin{equation}
\begin{aligned}
\fctvar(\alpha,\beta) = & \frac{1}{\cntvar (1-v)}
\begin{pmatrix}
v (\cntvar_1 (1-v) - \alpha) \\
v \alpha \\
(1 - v) (\cntvar_1 (1-v)-\alpha) \\
(1-v) \alpha \\
v (\cntvar_0 (1-v) - \beta) \\
v \beta \\
(1-v) ( \cntvar_0 (1-v) -\beta) \\
(1-v) \beta  \\
\end{pmatrix}, \\
& 0\le\alpha\le(1-v)\cntvar_1, \quad 0\le\beta\le(1-v)\cntvar_0. \\
\end{aligned}
\end{equation}
Furthermore, this family of solutions satisfies $\marginalmat \fctvar_0 = \marginalvec$ \iff{} $v = \posvar_0 / \cntvar_0 = \posvar_1 / \cntvar_1$, i.e.\ the base rates are equal and furthermore the score for both bins is equal to the base rate. 

\section{Proof of \Cref{thm:tr2}}
\label{sec:calib-proof}

\begin{proof}
Consider the product
\begin{equation}
    \begin{pmatrix}
    \fairmat_\textsc{pcb}\\
    \fairmat_\textsc{ncb}
    \end{pmatrix}
    z_0
    =
    \frac{M_1 N_0 - M_0 N_1}{N}
    \begin{pmatrix}
    \frac{v_0 v_1}{M_0 M_1} \\
 \frac{(1-v_0)(1-v_1)}{(M_0 - N_0)(M_1 - N_1)}
    \end{pmatrix}.
\end{equation}

This product equals the zero vector (and hence satisfies both PCB and NCB) \iff{} either of the conditions of the Corollary hold. (The last solution, $v_0 = 1$ and $v_1 = 0$, is inadmissible since $v_0 < v_1$ by assumption.)
\end{proof}

\section{Proof of \Cref{thm:cgdp}}
\label{sec:cgdp}

\begin{proof}
The result follows from solving
\begin{equation}
    \fairmat_\textsc{dp} \fctvar_0
    = \frac{\posvar_1 \cntvar_0 - \posvar_0 \cntvar_1}{N^2 (v_1 - v_0)}  = 0.
\end{equation}
\end{proof}

\section{Proof of \Cref{thm:cgpp}}
\label{sec:cgpp}

\begin{proof}
The result follows from solving
\begin{equation}
\small
\fairfn_\textsc{pp}(\fctvar_0) = v_1 (1-v_1) \left({(M_1 - N_1 v_0)}^2 - {(M_0 - N_0 v_0)}^2\right) = 0
\end{equation}
which is true \iff{} either condition in the Corollary is true. (The last case, $v_1 = 0$, is inadmissible by assumption.)
\end{proof}

In addition, here is a situation of fairness ``for free'',
in the sense that one notion of fairness automatically implies another.

\begin{corollary}
\label{thm:cgefor}
Consider a classifier that satisfies CG fairness.
Then, the classifier also satisfies EFOR fairness.
In other words, \{CG, EFOR\} is incompatible.
\end{corollary}

\begin{proof}
$\fairfn_\textsc{efor}(\fctvar_0) = 0$
vanishes identically.
\end{proof}

\section{Proof of \Cref{thm:chouldechova}}
\label{sec:proof-chould}

\begin{proof}
Finding the solution to $\fairfn_\textsc{pp}(\fctvar) = \fairfn_\textsc{efpr}(\fctvar) = \fairfn_\textsc{efnr}(\fctvar) = 0$ and also the linear system $\marginalmat\fctvar = \marginalvec$ yields the three conditions of the Theorem.
\end{proof}

\section{CG--accuracy trade-offs}%
\label{sec:cg-accuracy}

In the paper, we have only considered the case when $\lambda = \infty$ in the LAFOP: we only consider when the fairness criteria are satisfied exactly yielding several fairness--fairness trade-off results without heed to the accuracy of the classifiers. Nonetheless, recall that LAFOP allows us to express both fairness--accuracy and fairness--fairness trade-offs by introducing an accuracy objective along with a fairness regularizer. In this section, we show how the LAFOP can be used to theoretically analyze a simple fairness--accuracy trade-off. We present a small result that is relevant to the CG--accuracy trade-off considered in \cite{liu2019cg}.
\begin{theorem}%
\label{thm:liu}
Let $\baserate = (M_0 + M_1)/N$ be the base rate.
Consider a classifier that satisfies CG with $0\le v_0 < v_1 \le 1$.
Then, perfect accuracy is attained \iff{}
\begin{equation}
\frac{v_0 (1-2 v_1)}{1 - v_1 + v_0}
= \baserate \le \frac 1 8, \quad
\left\vert v_0 - \frac 1 4\right\vert \le \frac{\sqrt{1-8\alpha}} 4.
\end{equation}
\end{theorem}

\begin{proof}
The case of necessity ($\Rightarrow$) follows immediately from solving $\perfvec\cdot\fctvar_0 = 0$, where $\fctvar_0$ is defined in \Cref{thm:calib_incomp}.
The inequality conditions follow immediately from the constraint $0\le v_0 < v_1 \le 1$.
The case of sufficiency ($\Leftarrow$) follows immediately from \Cref{thm:calib_incomp} and substituting the equality condition.
\end{proof}

The condition of this theorem relates the scores $v_0$ and $v_1$ to the base rate of the data,
thus providing simple, explicit data dependencies that are necessary and sufficient.

\section{Experiment Details}
\label{sec:experiment_details}






\subsection{Optimization}
\label{sec:optimization_details}

For solving the optimization problems, we used solvers in the \texttt{scipy} package for Python \cite{scipy}. 
For linear fairness constraints, we used the simplex algorithm \cite{dantzig1963simplex},
and for other constrained optimization forms,
we used sequential least-squares programming (SLSQP) solver \cite{kraft1988slsqp,kraft1994slsqp}.

\subsection{Model-agnostic multi-way fairness--accuracy trade-offs}
\label{sec:multi}

We have only considered situations where zero or one parameter is sufficient to simultaneously specify the fairness strength for every fairness function, i.e.\ $\lmvar=\lmvar_0= \cdots= \lmvar_{K-1}$.
In this section, we generalize this and allow each regularization parameter to vary freely. It is then natural to consider the \MLAFOP{} (MLAFOP): $\argmin_{\fctvar\in\simplex}
    {(\perfvec \cdot \fctvar)}^2 + \sum_{i=0}^{K-1} \lmvar_i \Vert \fairmat^{(i)} \fctvar \Vert^2_2$,
where the regularization parameters $\lmvar_i$ now take different values across each of the $K$ fairness constraints. This allows for a general inspection of the individual effect of fairness constraints in a group.


\begin{figure}[ht!]
\centering
\includegraphics[width=0.5\textwidth]{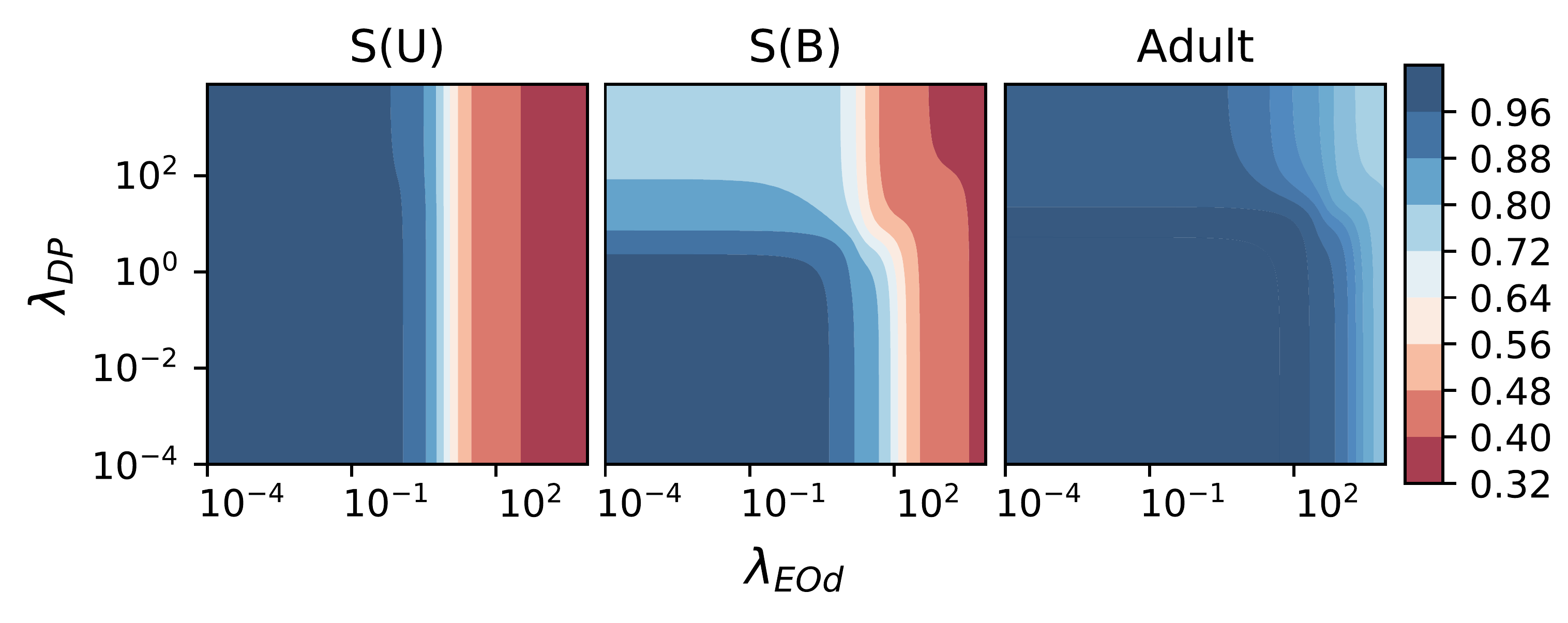}
\caption{Fairness--fairness--accuracy trade-off analysis using contour plot of accuracy with varying regularization strengths of Demographic Parity (DP) and Equalized Odds (EOd) for the unbiased synthetic dataset (left), biased synthetic dataset (middle), and Adult dataset (right). The contours show how the regularization strength of each fairness individually influence the accuracy ($1-\perfepsvar$) given the other (accuracy of 1.0 being the accuracy of the Bayes classifier). For the unbiased synthetic data, the accuracy change along the vertical axis (DP) is practically nonexistent given EOd, while along the horizontal axis (EOd) the change is drastic. Other datasets demonstrate more complex relationships.}%
\label{fig:multi_analysis}
\end{figure}

\begin{figure}[ht]
\centering
\includegraphics[width=0.7\textwidth]{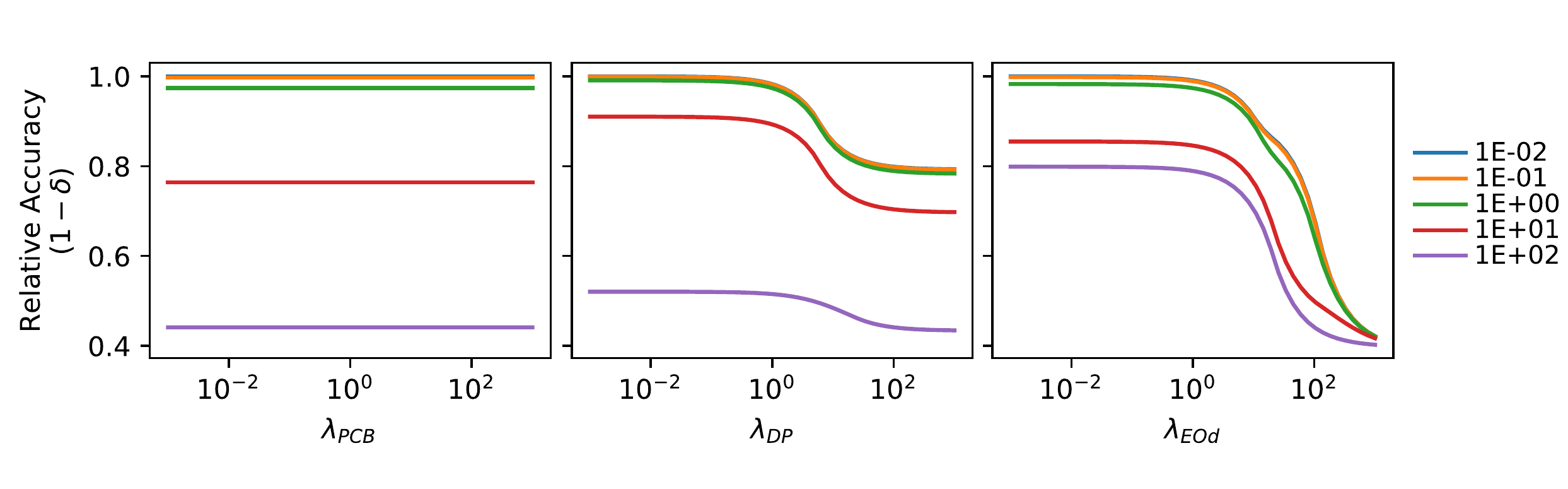}
\caption{The four-way trade-off between accuracy, PCB, EOd, and DP
in the biased synthetic dataset (\Cref{sec:datasets}).
Shown here is the ($1\!-\!\perfepsvar$) value as a function of some regularization strength $\lmvar_\fairfn$ for some fairness function $\fairfn$,
while holding all other $\lmvar_{\fairfn^\prime}$s constant (accuracy of 1.0 being the accuracy of the Bayes classifier).
The value next to each colored line in the legend represents constant values for the fixed $\lmvar_{\fairfn^\prime}$s.
Sweeping through PCB while keeping DP and EOd fixed (left) does not change the accuracy, whereas the other plots show multiple levels of variations.
For EOd (right), the accuracy levels converge quickly to the limiting value of 0.392 as shown in \Cref{fig:eps-delta-curve}, suggesting that the accuracy is more sensitive to changes in EOd constraint strength compared to the others.
}
\label{fig:slices}
\end{figure}

For instance, a three-way trade-off among EOd, DP, and accuracy can be visualized as a contour plot, similar to the one shown in \Cref{fig:multi_analysis}. And for general $(K+1)$-way trade-offs involving $K$ fairness constraints and accuracy, we visualize two-dimensional slices along the $K+1$-dimensional surface. For example, consider a four-way trade-off between a group of three fairness definitions (DP, EOd, PCB) and accuracy. \Cref{fig:eps-delta-curve} already showed that imposing PCB given (DP, EOd) does not affect $\perfepsvar$, which implies that PCB is the weakest in terms of its influence on $\perfepsvar$. To get more information, for the S(B) dataset, we show in \Cref{fig:slices} three cases of varying one $\lmvar$ for one fairness constraint while keeping the other $\lmvar$ values fixed in MLAFOP. Sweeping through PCB condition (left) does not affect $1-\perfepsvar$ at fixed EOd and DP levels, confirming the observation from \Cref{fig:eps-delta-curve}. Sweeping through DP conditions while keeping PCB and EOd strengths fixed (middle) results in a slight drop, but not big enough to make all levels to converge to values reported in \Cref{fig:eps-delta-curve} (0.392). Sweeping through EOd while keeping PCB and DP strengths fixed (right) on the other hand results in significant changes for all levels and convergence to the value 0.392, suggesting EOd is stronger than DP in terms of its influence on changing $\perfepsvar$. This notion of relative influence of fairness deserves further investigation, to see if these preliminary results are robust across other slices and datasets. Nonetheless, such analysis demonstrates a clear picture of how different notions of fairness interact with one another when they are to be imposed together.  

\newpage


\subsection{Connection to the post-processing methods for fair classification}
\label{sec:post-process-apdx}

We can explicitly rewrite the constraints in \eqref{eqn:hardt} using $\hat{z}$ and $\tilde{z}$, which respectively correspond to the fairness--confusion tensor of the given pre-trained classifier $\hat{Y}$ and the derived fair classifier $\tilde{Y}$:

\begin{align*}
    \gamma_0(\tilde{Y}) = \gamma_1(\tilde{Y}) &\Longleftrightarrow \fairmat_{\textsc{EOd}} \tilde{\fctvar} = 0 \\ 
    \gamma_0(\tilde{Y}) \in P_0(\hat{Y}) &\Longleftrightarrow \left(\frac{\tilde{\fctvar}_7}{\tilde{\fctvar}_7 + \tilde{\fctvar}_8}, \frac{\tilde{\fctvar}_5}{\tilde{\fctvar}_5 + \tilde{\fctvar}_6}\right) \in \\
    &\quad \text{convhull}\left\{(0,0), \left(\frac{\hat{\fctvar}_7}{\hat{\fctvar}_7 + \hat{\fctvar}_8}, \frac{\hat{\fctvar}_5}{\hat{\fctvar}_5 + \hat{\fctvar}_6}\right), \left(\frac{\hat{\fctvar}_8}{\hat{\fctvar}_7 + \hat{\fctvar}_8}, \frac{\hat{\fctvar}_6}{\hat{\fctvar}_5 + \hat{\fctvar}_6}\right), (1,1)\right\} \numberthis \label{eqn:postprocess2fact1} \\
    \gamma_1(\tilde{Y}) \in P_1(\hat{Y}) &\Longleftrightarrow \left(\frac{\tilde{\fctvar}_3}{\tilde{\fctvar}_3 + \tilde{\fctvar}_4}, \frac{\tilde{\fctvar}_1}{\tilde{\fctvar}_1 + \tilde{\fctvar}_2}\right) \in \\
    &\quad \text{convhull}\left\{(0,0), \left(\frac{\hat{\fctvar}_3}{\hat{\fctvar}_3 + \hat{\fctvar}_4}, \frac{\hat{\fctvar}_1}{\hat{\fctvar}_1 + \hat{\fctvar}_2}\right), \left(\frac{\hat{\fctvar}_4}{\hat{\fctvar}_3 + \hat{\fctvar}_4}, \frac{\hat{\fctvar}_2}{\hat{\fctvar}_1 + \hat{\fctvar}_2}\right), (1,1)\right\}
    \numberthis \label{eqn:postprocess2fact2}
\end{align*}
where the subscript $i$ of the fairness--confusion tensor corresponds to the $i$-th element in their vector representation as in \Cref{sec:linear}. By setting the objective function to be the classification error, imposing EOd fairness constraint and the model-dependent feasibility constraints in \eqref{eqn:postprocess2fact1} and \eqref{eqn:postprocess2fact2}, MS-LFAOP is the same optimization problem as the post-processing methods, now over the space of the fairness--confusion tensors. The FACT Pareto frontier obtained by solving MS-LAFOP therefore can assess the trade-off exhibited by any classifier post-processed in such ways.

In practice, the post-processing method solves \eqref{eqn:hardt} by parameterizing $\tilde{Y}$ with two variables for each group $a = 0,1$: ${\sf Pr}(\tilde{Y} = 1 | \hat{Y} = 1, A = a), {\sf Pr}(\tilde{Y} = 1 | \hat{Y} = 0, A = a)$. \cite{hardt2016equality}. These values are called the \emph{mixing rates}, as they indicate the probability of labels that should be flipped or kept for each group when post-processing the given classifier $\hat{Y}$. The algorithm then randomly selects the instances for each group to flip according to these mixing rates. These mixing rates can also be written in terms of the \fct{} $\tilde{z}$ and $\hat{z}$, by using the fact that 
\begin{align*}
{\sf Pr}(\tilde{Y} = \tilde{y} | Y = y, A = a) = {\sf Pr}(\tilde{Y} &= \tilde{y} | \hat{Y} = 1, A = a) {\sf Pr}(\hat{Y} = 1 | Y = y, A = a) + \\&{\sf Pr}(\tilde{Y} = \tilde{y} | \hat{Y} = 0, A = a) {\sf Pr}(\hat{Y} = 0 | Y = y, A = a),
\end{align*} and that ${\sf Pr}(\tilde{Y} = \tilde{y} | Y = y, A = a)$, ${\sf Pr}(\hat{Y} = \hat{y} | Y = y, A = a)$ terms are essentially what $\tilde{z}$ and $\hat{z}$ encode. Therefore, by using $\tilde{z}$ obtained from the MS-LAFOP above, we can compute the mixing rates to post-process the given classifier.

\end{document}